\documentclass[twoside,12pt]{article}
\usepackage{latexsym,amsthm,amssymb,amsmath}
\usepackage{url}
\usepackage{algorithm}
\usepackage{algorithmic}
\usepackage{stmaryrd}
\newenvironment{keywords}{\centerline{\bf\small
Keywords}\begin{quote}\small}{\par\end{quote}\vskip 1ex}
\def\citep{\cite}
\newtheorem{theorem}{Theorem}
\newtheorem{definition}[theorem]{Definition}
\newtheorem{corollary}[theorem]{Corollary}
\newtheorem{lemma}[theorem]{Lemma}
\newtheorem{proposition}[theorem]{Proposition}
\def\acks#1{\paragraph{Acknowledgements.} #1}

\renewcommand{\a}{\alpha}
\newcommand{\e}{{\rm e}}

\newcommand{\cM}{\mathcal{M}}
\newcommand{\cS}{\mathcal{S}}
\newcommand{\cC}{\mathcal{C}}

\newcommand{\cA}{\mathcal{A}}
\newcommand{\cB}{\mathcal{B}}
\newcommand{\cP}{\mathcal{P}}
\newcommand{\cL}{\mathcal{L}}
\newcommand{\cD}{\mathcal{D}}
\newcommand{\cU}{\mathcal{U}}
\newcommand{\labelop}[2]{\displaystyle_{#1}^{\hspace{-0.15em}(#2)}}

\title{\vspace{-4ex}
\vskip 2mm\bf\Large\hrule height5pt \vskip 4mm
Online Learning of k-CNF Boolean Functions
\vskip 4mm \hrule height2pt}
\author{
\begin{minipage}{0.45\textwidth}\center
{\bf Joel Veness}\\[3mm]
\normalsize DeepMind Technologies Ltd \\
\normalsize University of Alberta \\
\normalsize \texttt{joel@deepmind.com}
\end{minipage}
\hfill
\begin{minipage}{0.45\textwidth}\center
{\bf Marcus Hutter}\\[3mm]
\normalsize Research School of Computer Science \\
\normalsize Australian National University \\
\normalsize \texttt{http://www.hutter1.net/}
\end{minipage}
\vspace{2ex}
}
\date{March 2014}

\begin{document}
\maketitle

\begin{abstract}%
This paper revisits the problem of learning a k-CNF Boolean function from examples in the context of online learning under the logarithmic loss.
In doing so, we give a Bayesian interpretation to one of Valiant's celebrated PAC learning algorithms, which we then build upon to derive two efficient, online, probabilistic, supervised learning algorithms for predicting the output of an unknown k-CNF Boolean function.
We analyze the loss of our methods, and show that the cumulative log-loss can be upper bounded, ignoring logarithmic factors, by a polynomial function of the size of each example.
\end{abstract}

\begin{keywords}
  k-CNF, Online Learning, Logarithmic Loss
\end{keywords}

\section{Introduction}

In 1984, Leslie Valiant introduced the notion of Probably Approximately Correct (PAC) learnability, and gave three important examples of some non-trivial concept classes that could be PAC learnt given nothing more than a sequence of positive examples drawn from an arbitrary IID distribution \citep{Valiant:1984:TL:1968.1972}.
One of these examples was the class of $k$-CNF Boolean functions.
Valiant's approach relied on a polynomial time reduction of this problem to that of PAC learning the class of monotone conjunctions.
In this paper, we revisit the problem of learning monotone conjunctions from a different viewpoint.
This will allow us to derive two new online, probabilistic prediction algorithms that: (i) learn from both positive and negative examples; (ii) avoid making IID assumptions; (iii) suffer low logarithmic loss for arbitrary sequences of examples; (iv) run in polynomial time and space.

Our motivation for investigating this setting comes from our interest in building predictive ensembles for universal source coding.
In particular, we are interested in prediction methods that satisfy the following \emph{power} desiderata, i.e. methods which:
(p) make probabilistic predictions;
(o) are strongly online;
(w) work well in practice;
(e) are efficient;
(r) and have well understood regret/loss/redundancy properties.
Methods satisfying these properties can be combined in a principled fashion using techniques such as those discussed by \cite{VenessSH12,Mattern:2013:LGM:2495257.2495873}, giving rise to ensembles with clearly interpretable predictive capabilities.

Our contribution stems from noticing that Valiant's method can be interpreted as a kind of MAP model selection procedure with respect to a particular family of priors.
While this observation is in itself unremarkable, we show that this family of priors possess a number of desirable computational properties.
In particular, we show that given $n$ positive examples, it is possible to perform exact Bayesian inference over the $2^d$ possible hypotheses in time $O(nd)$ and space $O(d)$.
Unfortunately, these desirable computational properties do not seem to readily extend to the case where both positive and negative examples are presented.
A workable heuristic approach might be to directly use the efficient Bayesian predictor for prediction, only updating the posterior weights when positive instances are seen.
However, as well as needlessly throwing away potentially useful information, this approach also makes it impossible to provide meaningful loss guarantees that hold with respect to all input sequences.
Our first contribution is to introduce a hybrid algorithm, which uses a combination of Bayesian inference and memorization to construct a polynomial time algorithm whose loss is bounded by $O(d^2)$ for the class of monotone conjunctions.
Our second contribution is a more practical algorithm, which requires less space and whose loss is bounded by $O(d \log n)$.
Finally, similarly to Valiant, we describe how to combine our algorithms with a reduction that (for fixed $k$) enables the efficient learning of $k$-CNF Boolean functions from examples.

\section{Preliminaries}\label{sec:preliminaries}

We first introduce some notation to formalize our problem setting.

\paragraph{Notation.}
A Boolean variable $x$ is an element of $\cB := \{ \bot, \top \} = \{ 0, 1\}$.
We identify false $\bot$ with 0 and true $\top$ with 1,
since it allows us to use Boolean functions as likelihood functions for deterministically generated data.
We keep the boolean operator notation whenever more suggestive.
The unary not operator is denoted by $\neg$, and is defined as $\neg : 0 \mapsto 1;  1 \mapsto 0$ ($\neg x=1-x$).
The binary conjunction and disjunction operators are denoted by $\wedge$ and $\vee$ respectively, and are given by the maps $\wedge : (1,1) \mapsto 1; \text{~or~} 0 \text{~otherwise}$ ($x\wedge y = x\cdot y$).
and $\vee : (0,0) \mapsto 0; \text{~or~} 1 \text{~otherwise}$ ($x\vee y = \max\{x,y\}$).
A literal is a Boolean variable $x$ or its negation $\neg x$; a positive literal is a non-negated Boolean variable.
A clause is a finite disjunction of literals.
A monotone conjunction is a conjunction of zero or more positive literals. For example, $x_1 \wedge x_3 \wedge x_6$ is a monotone conjunction, while $\neg x_1 \wedge x_3$ is not.
We adopt the usual convention with conjunctions of defining the zero literal case to be  vacuously true.
The power set of a set $\cS$ is the set of all subsets of $\cS$, and will be denoted by $\cP(\cS)$.
For convenience, we further define $\cP_d := \cP(\{ 1, 2, \dots, d \})$.
We also use the Iverson bracket notation $\llbracket P \rrbracket$, which given a predicate $P$, evaluates to 1 if $P$ is true and 0 otherwise.
With our identification this is optional but useful syntactic sugar.
We also use the notation $x_{1:n}$ and $x_{<n}$ to represent the sequences of symbols $x_1 x_2 \dots x_n$ and $x_1 x_2 \dots x_{n-1}$ respectively.
Furthermore, base two is assumed for all logarithms in this paper.
Finally, we use the notation $a^i$ to index the $i$th component of a Boolean vector $a\in\cB^d$.

\paragraph{Problem Setup.}
We consider an online, sequential, binary, probabilistic prediction task with side information.
At each time step $t\in\mathbb{N}$, a $d$-dimensional Boolean vector of side information $a_t\equiv(a_t^1,...,a_t^d)\in\cB^d$ is presented to
a probabilistic predictor $\rho_t : \cB^d \to (\cB \to [0,1])$, which outputs a probability distribution over $\cB$.
A label $x_t\in\cB$ is then revealed, with the predictor suffering an instantaneous loss of $\ell_t := -\log \rho_t(x_t;a_t)$, with the cycle continuing ad infinitum.
It will also prove convenient to introduce the joint distribution $\rho(x_{1:n};a_{1:n})$, which lets us express the cumulative loss $\cL_n(\rho)$ in the form
\begin{equation*}
  \cL_n(\rho) ~:=~ \sum_{i=1}^n \ell_t  ~=~ -\log \prod_{t=1}^n \rho_t(x_t;a_t) ~=:~ -\log \rho(x_{1:n};a_{1:n})
\end{equation*}
We later use the above quantity to analyze the theoretical properties of our technique.
As is usual with loss or regret based approaches, our goal will be to construct a predictor $\rho$ such that $\cL_n(\rho) / n \to 0$ as $n \to \infty$ for an interesting class of probabilistic predictors $\cM$.
The focus of our attention for the remainder of this paper will be on the class of monotone conjunctions.

\paragraph{Brute force Bayesian learning.}
Consider the monotone conjunction $h_\cS(a_t):=\bigwedge_{i\in\cS} a_t^i$ for some $\cS\in\cP_d$, classifying $a_t\in\cB^d$ as $h_\cS(a_t)\in\cB$.
This can be extended to the function $h_\cS : \cB^{n \times d} \to \cB^n$ that returns the vector
$h_\cS(a_{1:n}):=(h_\cS(a_1),...,h_\cS(a_n))$.
One natural Bayesian approach to learning monotone conjunctions would be to place a uniform prior over the set of $2^d$ possible deterministic predictors that are monotone conjunctions of the $d$ Boolean input variables.
This gives the Bayesian mixture model
\begin{equation}\label{eq:monotone_mixture}
\xi_d(x_{1:n};a_{1:n}) ~:=~ \sum_{\cS\in\cP_d} \frac{1}{2^{d}} \nu_\cS(x_{1:n};a_{1:n}),~~~
\text{where}~~~ \nu_\cS(x_{1:n};a_{1:n}) ~:=~ \left\llbracket h_\cS(a_{1:n}) = x_{1:n} \right\rrbracket
\end{equation}
is the deterministic distribution corresponding to $h_\cS$.
Note that when $\cS=\{\}$, the conjunction $\bigwedge_{i\in\cS} a^i_t$ is vacuously true.
From here onwards, we will say hypothesis $h_\cS$ generates $x_{1:n}$ if $h_\cS(a_{1:n}) = x_{1:n}$.
For sequential prediction, the predictive probability $\xi_d(x_t|x_{<t};a_{1:t} )$ can be obtained by computing the ratio of the marginals, that is
$\xi_d(x_t|x_{<t};a_{1:t} ) = \xi_d(x_{1:t};a_{1:t} )~/~\xi_d(x_{<t};a_{<t} )$.
Note that this form of the predictive distribution is equivalent to using Bayes rule to explicitly compute the posterior weight for each $\cS$, and then taking a convex combination of the instantaneous predictions made by each hypothesis.

The loss of this approach for an arbitrary sequence of data generated by some $h_{\cS^*}$ for $\cS^*\in\cP_d$, can be upper bounded by
\begin{eqnarray*}
\cL_n(\xi_d) ~:=~ -\log \xi_d(x_{1:n};a_{1:n})
&=& -\log \sum_{\cS\in\cP_d} \tfrac{1}{2^{d}} \left\llbracket h_\cS(a_{1:n}) = x_{1:n} \right\rrbracket \\
&\leq&  -\log \tfrac{1}{2^{d}} \left\llbracket h_{\cS^*}(a_{1:n}) = x_{1:n} \right\rrbracket
~=~ d.
\end{eqnarray*}
Of course the downside with this approach is that a naive computation of Equation~\ref{eq:monotone_mixture} takes time $O(n \, 2^d)$.
Indeed one can show that no polynomial-time algorithm in $d$ for
$\xi_d$ exists (assuming P$\neq$NP):
\begin{theorem}[$\xi_d$ is \#P-complete]\label{thm:xidhard}
Computing the function $f:\{0,1\}^{n\times d}\to\{0,...,2^d\}$ defined as
$f(a_{1:n}):=2^d\xi_d(0_{1:n};a_{1:n})$ is \#P-complete.
\end{theorem}
We prove hardness by a two-step reduction: counting independent sets, known to be
\#P-hard, to computing the cardinality of a union of power sets to computing $\xi_d$:

\begin{definition}[UPOW]\label{def:upow}
Given a list of $n$ subsets $\cS_1,\dots,\cS_n$ of $\{1,\dots,d\}$,
compute $A := |\cP(\cS_1)\cup\dots\cup\cP(\cS_n)|$, i.e.\ the size of the union of the power sets of $\cS_1,\dots,\cS_n$.
\end{definition}

\begin{lemma}[UPOW$\to\xi_d$]\label{lem:xiupow}
If $a_t$ is defined as the $d$-dimensional characteristic bit vector describing
the elements in $\cS_t$, i.e.\ $a_t^i:=\llbracket i\in \cS_t\rrbracket$, then
$A = 2^d[1-\xi_d(0_{1:n}|a_{1:n})]$.
\end{lemma}

\begin{proof}
Since $h_\cS(a_t)\!=\!1$ iff $\cS\!\subseteq\!\cS_t$ iff $\cS\!\in\!\cP(\cS_t)$
we have
\begin{equation*}
  \llbracket h_\cS(a_{1:n})=0_{1:n} \rrbracket
  ~\iff~ \bigwedge_{t=1}^n [h_\cS(a_t)=0]
  ~\iff~ \neg\exists t:\cS\in\cP(\cS_t)
  ~\iff~ \cS\not\in\cP(\cS_1)\cup...\cup\cP(\cS_n)
\end{equation*}
which implies $\sum_{\cS\in\cP_d}\nu_\cS(0_{1:n}|a_{1:n})=2^d-A$.
\end{proof}

The intuition behind Lemma \ref{lem:xiupow} is that since $\xi_d$ uses a uniform prior over $\cP_d$, the number of hypotheses consistent with the data is equal to $2^d \xi_d(0_{1:n}|a_{1:n})$, and therefore the number of hypotheses inconsistent with the data is equal to $2^d[1-\xi_d(0_{1:n}|a_{1:n})]$.
One can easily verify that the set of hypotheses inconsistent with a single negative example is
$I_t := \cP \left( \left \{ i \in \{1,\dots, d\} \,:\, \llbracket a^i_t = 1 \rrbracket \right \} \right)$,
hence the set of hypotheses inconsistent with the data is equal to $\left | \cup_{t=1}^n I_t \right |$.

\begin{theorem}[IS$\to$UPOW, Brendan McKay, private communication]\label{thm:upowhard}
UPOW is \#P-hard.
\end{theorem}

\begin{proof}
Let $G=(V,E)$ be an undirected graph with vertices $V=\{1,...,d\}$
and edges $E=\{e_1,...,e_n\}$, where edges are $e=\{v,w\}$ with
$v,w\in V$ and $v\neq w$. An independent set $I$ is a set of
vertices no two of which are connected by an edge. The set of
independent sets is $\text{IS}:=\{I\subseteq V:\forall e\in E:
e\not\subseteq I\}$. It is known that counting independent sets,
i.e.\ computing $|\text{IS}|$ is \#P-hard \citep{Vadhan01}.

We now reduce IS to UPOW:
Define $\cS_t:=V\setminus e_t$ for $t\in\{1,...,n\}$ and consider any $W\subseteq V$ and its complement
$\overline{W}=V\setminus W$. Then
\begin{eqnarray*}
  & & W\not\in\text{IS}
  ~\iff~ \exists e\in E:e\subseteq W
  ~\iff~ \exists t: e_t\subseteq W
  ~\iff~ \exists t: \overline{W}\subseteq\cS_t
\\
  & & ~\iff~ \exists t: \overline{W}\in\cP(\cS_t)
  ~\iff~ \overline{W}\in\cP(\cS_1)\cup...\cup\cP(\cS_n)
\end{eqnarray*}
Since set-complement is a bijection and there are $2^{|V|}$ possible $W$, this implies
\begin{equation*}
  |\text{IS}| ~+~ |\cP(\cS_1)\cup...\cup\cP(\cS_n)| ~=~ 2^{|V|}
\end{equation*}
Hence an efficient algorithm for computing $|\cP(\cS_1)\cup...\cup\cP(\cS_n)|$
would imply the existence of an efficient algorithm for computing $|\text{IS}|$.
\end{proof}

\begin{proof}{\bf of Theorem~\ref{thm:xidhard}.}
Lemma~\ref{lem:xiupow} and Theorem~\ref{thm:upowhard} show that $f$ is \#P-hard.
What remains to be shown is that $f$ is in $\#P$.
First consider UPOW function $u:\cP_d^n\to\{0,...,2^d\}$ defined as $u(\cS_1,...,\cS_d):=A$.
With identification $\{0,1\}^d\cong\cP_d$ via
$a_t^i=\llbracket i\in\cS_t\rrbracket$ and $\cS_t=\{i:a_t^i=1\}$,
Lemma~\ref{lem:xiupow} shows
that $f(a_{1:n})+u(\cS_1,...,\cS_n)=2^d$.
Since $\cS\in\cP(\cS_1)\cup...\cup\cP(\cS_n)$ iff $\exists t:\cS\in\cP(\cS_t)$ iff $\exists t:\cS\subseteq\cS_t$,
the non-deterministic polynomial time algorithm ``Guess $\cS\in\cP_d$ and accept iff $\exists t:\cS\subseteq\cS_t$''
has exactly $A$ accepting paths, hence $u$ is in $\#P$.
Since this algorithm has $2^d$ paths in total,
swapping accepting and non-accepting paths shows that also $f$ is in $\#P$.
\end{proof}

One interesting feature of our reduction was that we only required a sequence of negative examples.
As we shall see in Section \ref{sec:monotone}, exact Bayesian inference is tractable if only positive examples are provided.
Finally, one can also show that the Bayesian predictor $\xi_d$ obtains the optimal loss.

\begin{proposition}
There exists a sequence of side information $a_{1:2^d} \in \cB^{2^d \times d}$ such that for any probabilistic predictor $\rho_t : \cB^d \to (\cB \to [0,1])$, there exists an $\cS \in \cP_d$ such that $h_{\cS}$ would generate a sequence of targets that would give $\cL_{2^d}(\rho) \geq d$.
\end{proposition}

\begin{proof}
Consider the sequence of side information $a_{1:2^d} \in \cB^{2^d \times d}$, where $a^i_d$ is defined to be the $i$th digit of the binary representation of $t$, for all $1 \leq t \leq 2^d$.
As
\begin{equation}
\left| \{ x_{1:2^d} \, : \, x_{1:2^d} \text{~is generated by an~} \cS \in \cP_d \} \right| = 2^d,
\end{equation}
to have $\cL_{2^d}(\rho) < \infty$ for all $x_{1:2^d}$, we need $\rho(x_{1:2^d}) > 0$ for each of the $2^d$ possible target strings, which implies that $\cL_{2^d}(\rho) \geq d$.
\end{proof}

\paragraph{Memorization.}
As a further motivating example, it is instructive to compare the exact Bayesian predictor to that of a naive method for learning monotone conjunctions that simply memorizes the training instances, without exploiting the logical structure within the class.
To this end, consider the sequential predictor that assigns a probability of
\begin{equation*}
m_d(x_n|x_{<n};a_{1:n}) = \left\{
     \begin{array}{lr}
       \llbracket x_n = l(a_{1:n}, x_{<n}) \rrbracket \text{~~if~~}  a_n\in\{ a_1, \dots, a_{n-1} \}; \\
       \tfrac{1}{2} \text{~~~~~~~~~~~~~~~~~~~~~~~~~~~~~~~~otherwise}
     \end{array}
   \right.
\end{equation*}
to each target, where $l(a_{1:n}, x_{<n})$ returns the value of $x_t$ for some $1 \leq t \leq n-1$ such that $a_n = a_t$.
Provided the data is generated by some $h_\cS$ with $\cS\in\cP_d$, the loss of the above memorization technique is easily seen to be at most $2^d$.
This follows since an excess loss of 1 bit is suffered whenever a new $a_k$ is seen, and there are at most $2^d$ distinct inputs (of course no loss is suffered whenever a previously seen $a_k$ is repeated).
While both memorization and the Bayes predictor suffer a constant loss that is independent of the number of training instances, the loss of the memorization technique is exponentially larger as a function of $d$.
This is offset by the fact that memorization can be implemented in essentially $O(nd)$ time by storing the previously seen examples in a hash table.
Later we show that by using a combination of Bayesian inference and a slightly more sophisticated form of memorization, it is possible to construct a method whose running time and loss are both bounded by polynomial functions of $n$ and $d$.

\section{Bayesian learning of monotone conjunctions from positive examples}\label{sec:monotone}

We now show how exact Bayesian inference over the class of monotone
conjunctions can be performed efficiently, provided learning only
occurs from positive examples $x_{1:n}\equiv 1_{1:n}$.
Using the generalized distributive law \citep{ali00} we derive an
alternative form of Equation~\ref{eq:monotone_mixture} that can be
straightforwardly computed in time $O(nd)$.

\begin{proposition}\label{prop:trick}
For all $n,d\in\mathbb{N}$, for all $a_{1:n}\in \cB^{n \times d}$, then
\begin{equation*}
\xi_d(1_{1:n};a_{1:n}) ~=~ \prod_{i=1}^d \left( \frac{1}{2} + \frac{1}{2} \left\llbracket\bigwedge_{t=1}^n a^i_t \right\rrbracket \right).
\end{equation*}
\end{proposition}
\begin{proof}
Consider what happens when the expression
$\prod_{i=1}^d \left( \frac{1}{2} + \frac{1}{2} \llbracket\bigwedge_{t=1}^n a^i_t\rrbracket  \right)$
is expanded.
We get a sum containing $2^d$ terms, that can be rewritten as
\begin{equation*}
\sum_{\cS\in\cP_d } \frac{1}{2^d} \left\llbracket \bigwedge_{i\in\cS} \bigwedge_{t=1}^n a^i_t \right\rrbracket
~=~ \sum_{\cS\in\cP_d } \frac{1}{2^d}  \left\llbracket h_\cS(a_{1:n}) = 1_{1:n} \right\rrbracket
~=~ \xi_d(1_{1:n};a_{1:n}).
\end{equation*}
where the second equality follows from Equation~\ref{eq:monotone_mixture} and the first one from
\begin{equation}\label{eq:trick}
\nu_\cS(1_{1:n}|a_{1:n})
~=~ \left\llbracket h_\cS(a_{1:n}) = 1_{1:n} \right\rrbracket
~=~ \bigwedge_{t=1}^n h_\cS(a_t)
~=~ \bigwedge_{t=1}^n \bigwedge_{i\in\cS} a_t^{i}
~=~ \bigwedge_{i\in\cS} \bigwedge_{t=1}^n  a_t^{i}
\end{equation}
\end{proof}

\paragraph{On MAP model selection from positive examples.}
If we further parametrize the right hand side of Proposition~\ref{prop:trick} by introducing a hyper-parameter $\a\in(0,1)$ to give
\begin{equation}\label{eq:alpha_factored_xi}
\xi^\a_d(1_{1:n};a_{1:n}) ~:=~ \prod_{i=1}^d \left( (1-\a) + \a \left\llbracket \bigwedge_{t=1}^n a^i_t \right\rrbracket \right),
\end{equation}
we get a family of tractable Bayesian algorithms for learning monotone conjunctions from positive examples.
The $\a$ parameter controls the bias toward smaller or larger formulas; smaller formulas are favored if $\a < \tfrac{1}{2}$, while larger formulas are favored if $\a > \tfrac{1}{2}$, with the expected formula length being $\a d$.
If we denote the prior over $\cS$ by $w_\a(\cS) := \a^{|\cS|}(1-\a)^{d-|\cS|}$, we get the mixture
\begin{equation*}
  \xi^\a_d(x_{1:n};a_{1:n}) ~=~ \sum_{\cS\in\cP_d} w_\a(\cS) \nu_\cS(x_{1:n};a_{1:n})
\end{equation*}
From this we can directly read off the maximum a posteriori (MAP) model
\begin{equation*}
\cS_n' ~:=~ \arg\max_{\cS\in\cP_d} w_\a(\cS|x_{1:n};a_{1:n})
~=~ \arg\max_{\cS\in\cP_d} w_\a(\cS)\nu_\cS(x_{1:n}|a_{1:n})
\end{equation*}
under various choices of $\a$. The second equality follows from Bayes rule.
For positive examples (i.e. $x_{1:n}=1_{1:n})$, Equation~\ref{eq:trick} allows us to rewrite this as
\begin{equation*}
\cS_n' ~=~ \arg\max_{\cS\in\cP_d} w_\a(\cS) \left\llbracket \bigwedge_{i\in\cS} \bigwedge_{t=1}^n a^i_t \right\rrbracket
\end{equation*}
For $\a > \tfrac{1}{2}$, the MAP model $\cS_n'$ at time $n$ is unique, and is given by
\begin{equation*}
\cS_n' ~=~ \left \{ i\in\{1, \dots, d \} : \bigwedge_{t=1}^n a^i_t \right \}.
\end{equation*}
For $\a=\tfrac{1}{2}$, a MAP model is any subset of $\cS_n'$.
For $\a<\tfrac{1}{2}$, the MAP model is $\{ \}$.
Finally, we remark that the above results allow for a Bayesian interpretation of Valiant's algorithm for PAC learning monotone conjunctions.
His method, described in Section 5 of \citep{Valiant:1984:TL:1968.1972}, after seeing $n$ positive examples, outputs the concept $\bigwedge_{i\in\cS'_n} x_i$; in other words, his method can be interpreted as doing MAP model selection using a prior belonging to the above family when $\a > \tfrac{1}{2}$.

\paragraph{A Heuristic Predictor.}
Next we discuss a heuristic prediction method that incorporates Proposition~\ref{prop:trick} to efficiently perform Bayesian learning on only the positive examples.
Consider the probabilistic predictor $\xi_d^+$ defined by
\begin{eqnarray}\label{eq:heuristic_predictor}
\xi_d^+(x_n|x_{<n};a_{1:n}) &:=& \frac{\xi_d(x^+_{<n}x_n;a^+_{<n} a_n)}{\xi_d(x^+_{<n};a^+_{<n})},
\end{eqnarray}
where we denote by $a^+_{<n}$ the subsequence of $a_{<n}$ formed by deleting the $a_k$ where $x_k = 0$, for $1 \leq k \leq n-1$.
Similarly, $x^+_{<n}$ denotes to the subsequence formed from $x_{<n}$ by deleting the $x_k$ where $x_k = 0$.
Note that since $\xi_d(x^+_{<n}0;a^+_{<n} a_n) = \xi_d(x^+_{<n};a^+_{<n}) \left(1 - \xi_d(1|x^+_{<n};a^+_{<n} a_n) \right)$, Equation \ref{eq:alpha_factored_xi} can be used to efficiently compute Equation~\ref{eq:heuristic_predictor}.
To further save computation, the values of the $\bigwedge_{t=1}^n a^i_t$ terms can be incrementally maintained using $O(d)$ space.
Using these techniques, each prediction can be made in $O(d)$ time.

Of course the main limitation with this approach is that it ignores all of the information contained within the negative instances.
It is easy to see that this has disastrous implications for the loss.
For example, consider what happens if a sequence of $n$ identical negative instances are supplied.
Since no learning will ever occur, a positive constant loss will be suffered at every timestep, leading to a loss that grows linearly in $n$.
This suggests that some form of memorization of negative examples is necessary; we will explore this further in the next section.

\paragraph{Discussion.}
There are some other noteworthy examples of where it is possible to efficiently perform exact Bayesian inference over large discrete spaces.
The Context Tree Weighting algorithm \citep{ctw95} performs exact model averaging over the space of bounded depth Prediction Suffix Trees \citep{Begleiter04onprediction}.
The switch distribution \citep{erven2007} performs exact model averaging over sequences of model indices, where each model index sequence describes a model formed by composing a particular sequence of indexed base models.
The Context Tree Switching algorithm \citep{DBLP:conf/dcc/VenessNHB12} combines these ideas to perform model averaging over all possible sequences of bounded depth Prediction Suffix Trees.
\cite{KoolenAW12} showed how to efficiently use model averaging to extend the ideas behind the switch distribution to robustly handle the case where only a small subset of models from a large pool are expected to make good probabilistic predictions.
Partition Tree Weighting \citep{veness13}, Live and Die Coding \citep{willemsPSMS97} and related methods for piecewise stationary sources \citep{Willems96,gyorgy2011efficient,Shamir99lowcomplexity} also work by efficiently performing model averaging over large discrete spaces defined by transition diagrams \citep{Willems96}.
The common theme amongst these techniques is the careful design of priors that allow for the application of either the generalized distributive law \citep{ali00} and/or dynamic programming to avoid the combinatorial explosion caused by naively averaging over a large number of models.

\section{An efficient, low loss algorithm for learning monotone \nobreak conjunctions}\label{sec:learning_from_both}

In this section we extend the ideas from the previous sections to construct an efficient online algorithm whose loss is bounded by $O(d^2)$.
The main idea is to extend the heuristic predictor so that it simultaneously memorizes negative instances while also favoring predictions of 0 in cases where the Bayesian learning component of the model is unsure.
The intuition is that by adding memory, there can be at most $2^d$ times where a positive loss is suffered.
Moving the $\a$ parameter closer towards $1$ causes the Bayesian component to more heavily weigh the predictions of the longer Boolean expressions consistent with the data, which has the effect of biasing the predictions more towards $0$ when the model is unsure.
Although this causes the loss suffered on positive instances to increase, we can show that this effect is relatively minor.
Our main contribution is to show that by setting $\a=2^{-d / 2^d}$, the loss suffered on both positive and negative instances is balanced in the sense that the loss can now be upper bounded by $O(d^2)$.
We proceed by first describing the algorithm, before moving on to analyze its loss.

\paragraph{Algorithm.}
The algorithm works very similarly to the previously defined heuristic predictor, with the following two modifications:
firstly, the set of all negative instances is incrementally maintained within a set $\cA$, with 0 being predicted deterministically if the current negative instance has been seen before; secondly, the $\xi_d$ terms in Equation~\ref{eq:heuristic_predictor} are replaced with $\xi^\a_d$, with $\a=2^{-d/2^d}$.
More formally,
\begin{eqnarray}\label{eq:bayes0_defn}
\zeta_d(x_t|x_{<t};a_{1:t}) &:=&
\left\{
\begin{array}{cl}
  1-x_t & \text{if}~a_t\in\cA \\
  \frac{\xi^\a_d(x^+_{<t}x_t;a^+_{<t}a_t)}{\xi^\a_d(x^+_{<t};a^+_{<t})} & \text{otherwise}
\end{array}
\right.
\end{eqnarray}

\algsetup{indent=2em}
\begin{algorithm}[t!]
\caption{{$\zeta_d(x_{1:n};a_{1:n} )$}}\label{alg:monotone}
\begin{algorithmic}[1]
\medskip
\STATE $w_i \leftarrow 1 \text{~for~} 1 \leq i \leq d$
\STATE $\cA \leftarrow \{ \}$
\STATE $\a \leftarrow 2^{-d/2^d}$
\STATE $r \leftarrow 1$
\medskip
\FOR{$t=1$ to $n$}
		
		\medskip
		\STATE \text{Observe $a_t$}
		\medskip
		\IF{$a_t\in\cA$}
		    \STATE $p_t(1;a_t) \leftarrow 0$ \label{alg1anotinA}
		    \STATE $p_t(0;a_t) \leftarrow 1$
		\ELSE
		    \STATE $p_t(1;a_t) \leftarrow \prod_{i=1}^{d} \frac{(1-\a) + \a w_i a^i_t}{(1-\a) + \a w_i}$ \label{alg1ainA}
		    \STATE $p_t(0;a_t) \leftarrow 1 - p_t(1;a_t)$
		\ENDIF
		\medskip
		\STATE \text{Observe $x_t$ and suffer a loss of~} $-\log p_t(x_t;a_t).$
    \medskip
    \IF{$x_t=1$}
        \FOR{$i=1$ to $d$}
            \STATE $w_{i} \leftarrow w_i~a_t^{i}$
        \ENDFOR
    \ELSE
		    \STATE $\cA \leftarrow \cA \cup \{ a_t \}$
		\ENDIF
    \medskip
		\STATE $r \leftarrow p_t(x_t;a_t) r$
\ENDFOR
\medskip
\RETURN $r$
\end{algorithmic}
\end{algorithm}

Complete pseudocode is given in Algorithm~\ref{alg:monotone}.
The algorithm begins by initializing the weights and the set of negative instances $\cA$.
Next, at each time step $t$, a distribution $p_t(\cdot;a_t) $ over $\{0,1\}$ is computed.
If $a_t$ has previously been seen as a negative example, the algorithm predicts 0 deterministically.
Otherwise it makes its prediction using the previously defined Bayesian predictor (with $\a=2^{-d/2^d}$) that is trained from only positive examples.
The justification for Line~\ref{alg1ainA} is as follows: First note that
$w_i$ is always equal to the conjunction of the $i$th component of the inputs corresponding to the positive examples occurring before time $t$,
formally
\begin{equation*}
  w_i ~=~ \bigwedge_{\hspace{-2ex}\tau=1:a_\tau\not\in\cA\hspace{-2ex}}^{t-1} a_\tau
  ~~~~\text{which by Equation~\ref{eq:alpha_factored_xi} implies}~~~~ \xi^\a_d(x^+_{<t};a^+_{<t}) ~=~ \prod_{i=1}^{d} [(1-\a) + \a w_i]
\end{equation*}
Similarly $\xi^\a_d(x^+_{<t}1;a^+_{<t}a_t) = \prod_{i=1}^{d} [(1-\a) + \a w_i a_t^i]$,
which by Equation~\ref{eq:bayes0_defn} for $a_t\notin\cA$ implies
\begin{equation*}
  \zeta_d(x_t=1|x_{<t};a_{1:t}) ~=~ {\prod_{i=1}^{d} [(1-\a) + \a w_i a_t^i]\over \prod_{i=1}^{d} [(1-\a) + \a w_i]} ~=~ p_t(1;a_t)
\end{equation*}
Trivially $p_t(x_t;a_t)=1-x_t=\zeta_d(x_t|x_{<t};a_{1:t})$ for $a_t\in\cA$ from Line~\ref{alg1anotinA}.
After the label is revealed and a loss is suffered, the algorithm either updates $\cA$ to remember the negative instance or updates its weights $w_i$, with the cycle continuing.
The algorithm returns $r=\prod_{t=1}^n p_t(x_1;a_t)=\zeta_d(x_{1:n};a_{1:n})$ as claimed.
It requires $O(nd)$ space and processes each example in $O(d)$ time.

\paragraph{Analysis.}
We now analyze the cumulative log-loss when using $\zeta_d$ in place of an arbitrary monotone conjunction corresponding to some $\cS^*\in\cP_d$.
We begin by first proving a lemma, before using it to upper bound the loss of Algorithm~\ref{alg:monotone}.
\begin{lemma}\label{lem:magic}
For all $d\in\mathbb{N} \setminus \{ 1 \} $, we have that
$
-\log \left( 1 - 2^{-d/2^d} \right) \leq d.
$
\end{lemma}
\begin{proof}
We have that
\begin{equation}\label{pr:magic}
 -\ln(1-\e^{-d/\e^d})
 ~\leq~ -\ln \left( 1-{1\over 1+d/\e^d} \right)
 ~=~ \ln{1+d/\e^d\over d/\e^d}
 ~=~ d + \ln \left( {1\over d}+{1\over\e^d} \right)
 ~\leq~ d.
\end{equation}
The first bound follows from $\e^{-x}\leq{1\over 1+x}$. The equalities are simple algebra.
The last bound follows from ${1\over d}+{1\over\e^d}\leq 1$ for $d\geq 2$.
(A similar lower bound
$-\ln(1-\e^{-d/\e^d})\geq-\ln(1-(1-d/\e^d))=d-\ln d$ shows that the bound is rather tight for large $d$).
Substituting $d\leadsto d\ln 2$ in (\ref{pr:magic}) and dividing by $\ln 2$ proves the lemma.
\end{proof}

\vspace{1em}

\begin{theorem}\label{thm:main_result}
If $x_{1:n}$ is generated by a hypothesis $h_{\cS^*}$ such that $\cS^*\in\cP_d$ then for all $n\in\mathbb{N}$, for all $d\in\mathbb{N}\setminus\{1\}$, for all $x_{1:n}\in\cB^n$, for all $a_{1:n}\in \cB^{n \times d}$, we have that $\cL_n(\zeta_d) \leq 2d^2$.
\end{theorem}
\begin{proof}
We begin by decomposing the loss into two terms, one for the positive and one for the negative instances.
\begin{eqnarray}
\cL_n(\zeta_d) &=& \sum_{t=1}^n -\log \zeta_d(x_t|x_{<t};a_{1:t}) \notag \\
&=& \hspace{-0.85em} \sum_{\substack{t\in[1,n] \\\text{s.t.~} x_t=1 }} \hspace{-0.85em} -\log \zeta_d(x_t = 1|x_{<t};a_{1:t}) \hspace{0.65em} +  \sum_{\substack{t\in[1,n] \\\text{s.t.~} x_t = 0}} \hspace{-0.85em} -\log \zeta_d(x_t = 0|x_{<t};a_{1:t}) \notag \\
&=& \hspace{-0.85em} \sum_{\substack{t\in[1,n] \\\text{s.t.~} x_t=1 }} \hspace{-0.85em} -\log \frac{\xi^\a_d(x^+_{1:t};a^+_{1:t})}{\xi^\a_d(x^+_{<t};a^+_{<t})} \hspace{0.65em} + \sum_{\substack{t\in[1,n] \\\text{s.t.~} x_t = 0}} \hspace{-0.85em} -\log \zeta_d(x_t = 0|x_{<t};a_{1:t}) \notag \\
&=& -\log  \xi^\a_d(x^+_{1:n};a^+_{1:n}) \hspace{0.65em} + \sum_{\substack{t\in[1,n] \\\text{s.t.~} x_t = 0}} \hspace{-0.85em} -\log \zeta_d(x_t = 0|x_{<t};a_{1:t}), \label{eq:bayes0_left_right_split}
\end{eqnarray}
where we have used the notation $[1,d] := \{ 1, 2, \dots, d \}$.
The final step follows since the left summand telescopes.
Next we will upper bound the left and right terms in Equation \ref{eq:bayes0_left_right_split} separately.

For $\a\in(0.5,1)$, we have for the left term that
\begin{eqnarray}
-\log  \xi^\a_d(x^+_{1:n};a^+_{1:n})
&\leq& -\log \left ( \a^{|\cS^*|} (1-\a)^{d-|\cS^*|} \right) \notag \\
&\leq& -d \log (1 -\a). \label{eq:bayes_alpha_bound}
\end{eqnarray}

Now, let $\cU := \left \{ t\in[1,n] : x_t = 0 ~\wedge~ \bigwedge_{i=1}^{t-1} (a_t \neq a_i) \right \}$ denote the set of time indices where a particular negative instance is seen for the first time and let
\begin{equation}\label{def:Dn}
\cD_t ~:=~ \left \{ i\in[1,d] : \bigwedge_{\tau=1}^{t-1} \left( \neg x_\tau \vee a_\tau^{i} \right) \right\}
\end{equation}
denote the indices of the variables not ruled out from the positive examples occurring before time $t$.
Given these definitions, we have that
\begin{eqnarray}\label{eq:restricted_marginal}
\xi^\a_d(x^+_{<t};a^+_{<t}) \notag
&=& \sum_{\cS\in\cP_d} \a^{|\cS|} (1-\a)^{d-|\cS|} \left\llbracket h_\cS(a^+_{<t}) = x^+_{<t} \right\rrbracket
~=~ \sum_{\cS\in\cP(\cD_t)} \a^{|\cS|} (1-\a)^{d-|\cS|}
\\
&=& (1-\a)^{d-|\cD_t|}\sum_{\cS\in\cP(\cD_t)} \a^{|\cD_t|} (1-\a)^{|\cD_t|-|\cS|}
~=~ (1-\a)^{d-|\cD_t|}
\end{eqnarray}
and similarly for $t\in\cU$
\begin{eqnarray}\label{eq:restricted_marginal0}
\xi^\a_d(x^+_{<t}0;a^+_{<t}a^+)
&=& \sum_{\cS\in\cP_d} \a^{|\cS|} (1-\a)^{d-|\cS|} \left\llbracket h_\cS(a^+_{<t}a_t) = x^+_{<t}0 \right\rrbracket
\\  \notag
&=& \sum_{\cS\in\cP(\cD_t)} \a^{|\cS|} (1-\a)^{d-|\cS|} \left\llbracket h_\cS(a_t) = 0 \right\rrbracket
~\geq~ \a^{|\cD_t|}(1-\a)^{d-|\cD_t|}
\end{eqnarray}
The last inequality follows by dropping all terms in the sum except
for the term corresponding the maximally sized conjunction
$\bigwedge_{t\in\cD_t} x_t $, which must evaluate to 0 given $a_t$,
since $\cS^* \subseteq \cD_t$ and $t\in\cU$.
Using the above, we can now upper bound the right term
\begin{eqnarray}
& & \hspace{-5em} \sum_{\substack{t\in[1,n] \\\text{s.t.~} x_t = 0}} \hspace{-0.85em}
-\log \zeta_d(x_t = 0|x_{<t};a_{1:t})
~\labelop{=}{a}~ \sum_{t\in\cU} -\log{\xi^\a_d(x^+_{<t}0;a^+_{<t}a^+) \over \xi^\a_d(x^+_{<t};a^+_{<t}) } \notag \\
&\labelop{\leq}{b}& \sum_{t\in\cU} -\log \a^{|\cD_t|}
~\labelop{\leq}{c}~ \sum_{t\in\cU} -d \log \a
~\labelop{\leq}{d}~ -d \, 2^d \log \a  \label{eq:memo_upper_bound}.
\end{eqnarray}
Step $(a)$ follows from the definition of $\zeta_d$ and $\cU$ (recall that a positive loss occurs only the first time an input vector is seen).
Step $(b)$ follows from Equations~\ref{eq:restricted_marginal} and \ref{eq:restricted_marginal0}.
Step $(c)$ follows since $|\cD_t| \leq d$ by definition.
Step $(d)$ follows since there are at most $2^d$ distinct Boolean vectors of side information.

Now, by picking $\a=2^{-d/2^d}$, we have from Equation~\ref{eq:bayes_alpha_bound} and Lemma~\ref{lem:magic} that $$-\log  \xi^\a_d(x^+_{1:n};a^+_{1:n}) \leq d^2.$$
Similarly, from Equation~\ref{eq:memo_upper_bound} we have that
\begin{equation*}
\sum_{\substack{t\in[1,n] \\\text{s.t.~} x_t = 0}} \hspace{-0.85em} -\log \zeta_d(x_t = 0|x_{<t};a_{1:t}) \leq -d \, 2^d \log 2^{-d/2^d}= d^2.
\end{equation*}
Thus by summing our previous two upper bounds, we have that
\begin{equation*}
\cL_n(\zeta_d) ~=~ -\log  \xi^\a_d(x^+_{1:n};a^+_{1:n}) \hspace{0.65em} + \sum_{\substack{t\in[1,n] \\\text{s.t.~} x_t = 0}} \hspace{-0.85em} -\log \zeta_d(x_t = 0|x_{<t};a_{1:t})  \leq 2 d^2.
\end{equation*}
\end{proof}

\paragraph{Discussion.}
So far we have always made the assumption that the targets are generated by some unknown monotone conjunction.
If this assumption does not hold, it is possible to observe a label which will have a probability of 0 under $\zeta_d$, which will cause an infinite loss to be suffered.
The next section will present an algorithm which, as well as being more space efficient, avoids this problem.

\section{A more practical approach}

Although the loss of Algorithm~\ref{alg:monotone} is no more than $2d^2$ (and independent of $n$), a significant practical drawback is its $O(nd)$ space complexity.
We now present an alternative algorithm which reduces the space complexity to $O(d)$, at the small price of increasing the worst case loss to no more than $O(d \log n)$.

\paragraph{Algorithm.}
The main intuition for our next algorithm follows from the loss analysis of Algorithm~\ref{alg:monotone}.
Our proof of Theorem~\ref{thm:main_result} led to a choice of $\a=2^{-d/2^d}$, which essentially causes each probabilistic prediction to be largely determined by the prediction made by the longest conjunction consistent with the already seen positive examples.
This observation led us to consider Algorithm~\ref{alg:minplus}, which uses a smoothed of this.
More formally,
\begin{equation*}
\pi_d(x_t|x_{<t};a_{1:t}) ~:=~ \frac{t}{t+1} \left\llbracket \bigwedge_{i\in\cD_t} a^i_t = x_t \right\rrbracket + \frac{1}{t+1} \left\llbracket \bigwedge_{i\in\cD_t} a^i_t \neq x_t \right\rrbracket;
\end{equation*}
where $\cD_t$ denotes the indices of the variables not ruled out from the positive examples occurring before time $n$ as defined in Equation~\ref{def:Dn}.

Complete pseudocode for implementing this procedure in $O(d)$ time per iteration, using $O(d)$ space, is given in Algorithm~\ref{alg:minplus}.
The set $\cD$ incrementally maintains the set $\cD_t$.
Compared to Algorithm~\ref{alg:monotone}, the key computational advantage of this approach is that it doesn't need to remember the negative instances.

\algsetup{indent=2em}
\begin{algorithm}[t!]
\caption{{$\pi_{d}(x_{1:n};a_{1:n} )$}}\label{alg:minplus}
\begin{algorithmic}[1]
\medskip
\STATE $\cD \leftarrow \{1, 2, \dots, d\}$
\STATE $r \leftarrow 1$
\medskip
\FOR{$t=1$ to $n$}
		
		\medskip
		\STATE \text{Observe $a_t$}
		\medskip
		
		\IF{$\prod_{i\in\cD} a^i_t = 1$}
		    \STATE $p_t(1;a_t) := t/(t+1)$
				\STATE $p_t(0;a_t) := 1/(t+1)$
		\ELSE
				\STATE $p_t(1;a_t) := 1/(t+1)$
				\STATE $p_t(0;a_t) := t/(t+1)$
		\ENDIF
		
		\medskip
		\STATE \text{Observe $x_t$ and suffer a loss of~} $-\log p_t(x_t;a_t).$
    \medskip
    \IF{$x_t=1$}
						\medskip
            \STATE $\cD \leftarrow \cD \setminus \left\{ i\in\{1,\dots,d\} \,:\, a_t^{i} = 0 \right \}$
		\ENDIF
    \medskip
		\STATE $r \leftarrow p_t(x_t;a_t) \, r$
\ENDFOR
\medskip
\RETURN $r$
\end{algorithmic}
\end{algorithm}

\paragraph{Analysis.}
We next upper bound the loss of Algorithm~\ref{alg:minplus}.

\begin{theorem}\label{thm:minplus_bound}
If $x_{1:n}$ is generated by a hypothesis $h_{\cS^*}$ such that $\cS^*\in\cP_d$, then for all $n\in\mathbb{N}$, for all $d\in\mathbb{N}$, for all $x_{1:n}\in\cB^n$, for all $a_{1:n}\in \cB^{n \times d}$, we have that $\cL_n(\pi_d) \leq (d+1) \log (n+1)$.
\end{theorem}
\begin{proof}
As $x_{1:n}$ is generated by some $h_{\cS^*}$ where $\cS^*\in\cP_d$, we have that
$\cL_n(\pi_d)=-\nobreak\log \pi_d( x_{1:n};a_{1:n}).$
We break the analysis of this term into two cases.
At any time $1 \leq t \leq n$, we have either:
Case (i):	$\bigwedge_{i\in\cA} a^i_t = \top$, which implies $h_\cS(a_t)=1$ for all $\cS \subseteq \cA = \cD_t$.
As the data is generated by some $h_{S^*}$, we must have $\cS^* \subseteq \cD_t$ and therefore $x_t=1$, so a loss of $-\log \tfrac{t}{t+1}$ is suffered.
Case (ii):
$\bigwedge_{i\in\cA} a^i_t = \bot$, where one of two situations occur: a) if $x_t=0$ we suffer a loss of $-\log \tfrac{t}{t+1}$; otherwise b) we suffer a loss of $-\log(1/(t+1))=\log (t+1)$ and at least one element in $\cA$ gets removed.
Notice that as the set $\cA$ is initialized with $d$ elements, case b) can only occur at most $d$ times given any sequence of data.

Finally, notice that Case (ii b) contributes at most at $d$ times $\log(n+1)$ to the loss.
On the other hand, $\log\tfrac{t+1}{t}$ is suffered for each $t$ of case (i) and (ii a),
which can be upper bounded by $\sum_{t=1}^n\log\tfrac{t+1}{t} = \log(n+1)$.
Together they give the desired upper bound $(d+1)\log(n+1)$.
\end{proof}

We also remark that Algorithm~\ref{alg:minplus} could have been defined so that $p_t(1;a_t)=1$ whenever $\bigwedge_{i\in\cA} a^i_t = 1$.
The reason we instead predicted $1$ with probability $\tfrac{t}{t+1}$ is that it allows Algorithm~\ref{alg:minplus} to avoid suffering an infinite loss if the data is not generated by some monotone conjunction, without meaningfully affecting the loss analysis.

\section{Online learning of $k$-CNF Boolean functions}

Finally, we describe how our techniques can be used to probabilistically predict the output of an unknown $k$-CNF function.
Given a set of $d$ variables $\{ x_1, \dots, x_d  \}$, a $k$-CNF Boolean function is a conjunction of clauses $c_1 \wedge c_2 \wedge \dots \wedge c_m$, where for $1 \leq y \leq m$, each clause $c_y$ is a disjunction of $k$ literals, with  each literal being an element from $\{ x_1, \dots, x_d, \neg x_1, \dots, \neg x_d  \}$.
The number of syntactically distinct clauses is therefore $(2d)^k$.
We will use the notation $\cC^k_d$ to denote the class of $k$-CNF Boolean formulas that can be formed from $d$ variables.

The task of probabilistically predicting a $k$-CNF Boolean function of $d$ variables can be reduced to that of probabilistically predicting a monotone conjunction over a larger space of input variables.
We can directly use the same reduction as used by \cite{Valiant:1984:TL:1968.1972} to show that the class of $k$-CNF Boolean functions is PAC-learnable.
The main idea is to first transform the given side information $a\in\cB^d$ into a new Boolean vector $c\in\cB^{(2d)^k}$, where each component of $c$ corresponds to the truth value for each distinct $k$-literal clause formed from the set of input variables $\{ a^i \}_{i=1}^d$, and then run either Algorithm~\ref{alg:monotone} or Algorithm~\ref{alg:minplus} on this transformed input.
In the case of Algorithm~\ref{alg:monotone}, this results in an online algorithm where each iteration takes $O(d^k)$ time; given $n$ examples, the algorithm runs in $O(n d^k)$ time and uses $O(n d^k)$ space.
Furthermore, if we denote the above process using either Algorithm~\ref{alg:monotone} or Algorithm~\ref{alg:minplus} as $\text{\sc Alg1}^k_d$ or $\text{\sc Alg2}^k_d$ respectively, then Theorems~\ref{thm:main_result} and~\ref{thm:minplus_bound} allows us to upper bound the loss of each approach as follows.

\begin{corollary}\label{thm:loss}
For all $n \in\mathbb{N}$, for all $k \in\mathbb{N}$, for any sequence of side information $a_{1:n}\in\cB^{n \times d}$, if $x_{1:n}$ is generated from a hypothesis $h^*\in\cC^k_d$, the loss of $\text{{\sc Alg1}}^k_d$ and $\text{{\sc Alg2}}^k_d$ with respect to $h^*$ satisfies the upper bounds $\cL_n(\text{{\sc Alg1}})  \leq 2^{2k+1} d^{2k}$ and $\cL_n(\text{{\sc Alg2}})  \leq \left( 2^k d^k + 1 \right) \log (n+1)$ respectively.
\end{corollary}

It is also possible to provide similar reductions that can probabilistically predict Boolean functions formed from conjunctions or disjunction in $O(d)$ time and space: in the case of conjunctions, we expand the side information to also include $d$ additional terms, each of which corresponds to the negation of each $a_t^{i}$; disjunctions are handled by using De Morgan's law by first flipping the side information components and $x_t$, and then using the reduction technique for conjunctions.

\section{Experimental results}

We additionally performed some simulations to better understand our methods.
The source code to replicate all reported results can be found at: \url{jveness.info/software/kcnf.zip}.

\paragraph{Tightness of Theorem~\ref{thm:main_result} and Theorem~\ref{thm:minplus_bound}.}
We first investigated the tightness of the upper bounds given in Theorem~\ref{thm:main_result} and Theorem~\ref{thm:minplus_bound}.
The results are summarized in Table~\ref{tbl:results}.
For every $d\in\{ 2^k \}_{k=1}^8$, a random $\cS\in\cP_d$ was generated by first sampling a $\theta\in[0,1]$ uniformly at random, and then sampling a hypothesis corresponding to $\cS\in\cP_d$ distributed according to the probability mass function $f(\cS) := \theta^{|\cS|} (1 - \theta)^{d-{|\cS|}}$.
Next a Boolean vector of side information was generated by sampling from the uniform distribution over $\cB^d$, with each algorithm making a prediction, suffering a loss, and then seeing the correct label, for a total of $n=8192$ time steps.
This whole process was then repeated 1000 times, with the reported empirical loss corresponding to the maximum loss obtained across each of the 1000 repeats, for all $d$.
Table~\ref{tbl:results} suggests that the upper bound on the loss given in Theorem~\ref{thm:main_result} is reasonably tight, typically only 2-3 times higher than the highest observed empirical loss in our experiments.

\begin{table}[h!]
	\centering
		\begin{tabular}{| l || c | c | c | c | c | c | c | c |}
		\hline
		& 2 & 4 & 8 & 16 & 32 & 64 & 128 & 256 \\
		\hline
		Algorithm~\ref{alg:monotone} & 3.5 & 10.6 & 44.4 & 200.462 & 808.9 & 3745.8 & 15555.7 & 63623.4 \\
		Theorem~\ref{thm:main_result} & 8 & 32 & 128 & 512 & 2048 & 8192 & 32768 & 131072 \\
		\hline
		Algorithm~\ref{alg:minplus} & 11.9 & 19.5 & 36.8 & 75.8 & 97.1 & 98.4 & 111.3 & 116.0 \\
		Theorem~\ref{thm:minplus_bound} & 39 & 65 & 117 & 221 & 429 & 845 & 1677 & 3341 \\
		\hline
		\end{tabular}
	\caption{Empirical loss (in bits) of Algorithm~\ref{alg:monotone} and Algorithm~\ref{alg:minplus} versus their theoretical upper bounds for various dimensions (specific dimensions indicated in the column headings).}
	\label{tbl:results}
	\vspace{-0.25em}
\end{table}

\paragraph{Comparison of methods.}
Additionally, we compared Algorithms~\ref{alg:monotone} and~\ref{alg:minplus} to the work of \cite{Dash02exactmodel}, which uses an efficient model averaging technique to combine the outputs of $2^d$ different discriminative Naive Bayesian classifiers in $O(d)$ time per step; we give a brief overview of the technique as applied to our setting in Appendix \ref{app:madnb}.
We denote this method by MADNB (Model Averaging Discriminative Naive Bayes).
We used a similar setup as before, but also varied the number of timesteps $n$.
The results are shown in Table \ref{tbl:madnb_results}.
While it is difficult to draw too many conclusions, it does appear that the loss of MADNB scales linearly in our setting.

\begin{table}[t!]
	\centering
		\begin{tabular}{| l || c | c | c | c | c | c | c | c |}
		\hline
		& 2 & 4 & 8 & 16 & 32 & 64 & 128 & 256 \\
		\hline
		$n = 2048$  & 265.4 & 263.9 & 251.5 & 259.0 & 288.3 & 302.4 & 390.2 & 466 \\
		$n = 4096$  & 494.8 & 507.0 & 475.7 & 484.5 & 490.1 & 489.2 & 534.5 & 616.0 \\
		$n = 8192$  & 933.8 & 956.8 & 926.2 & 961.3 & 921.9 & 927.4 & 809.4 & 1133.43 \\
		$n = 16384$ & 1817.0 & 1808.4 & 1810.3 & 1807.8 & 1790.2 & 1791.7 & 1821.7 & 1993.6 \\
		$n = 32768$ & 3549.0 & 3554.5 & 3528.8 & 3538.3 & 3496.8 & 3562.8 & 3646.1 & 3821.8 \\
		\hline
		\end{tabular}
	\caption{Empirical loss of MADNB (in bits).}
	\label{tbl:madnb_results}
	\vspace{-1em}
\end{table}

Furthermore, we compared MADNB to Algorithm \ref{alg:minplus} using the $k$-CNF reduction on two binary classification tasks, each taken from the UCI repository \citep{BacheLichman2013}.
For the first task we used the Mushroom Dataset, and for the second we used King-Rook vs King Pawn chess endgames.
These data sets were chosen by filtering the list of available datasets to those that dealt with binary classification tasks, containing only categorical input features, and who had more than 1000 instances.
We were also careful to check that the results were qualitatively the same under different permutations of the instances and switching of the class labels.
The results are summarized in Tables \ref{tbl:uci_results} and \ref{tbl:uci_results2}.
Here we see two main effects.
First is that the real world performance of MADNB is much better than the worst case results presented earlier.
Secondly, we see the benefit of expanding the model class by increasing $k$; it seems that for both datasets, $k=3$ is sufficient for achieving superior or comparable predictive accuracy to MADNB.

\begin{table}[h!]
	\centering
		\begin{tabular}{| l || c | c | c | c |}
		\hline
		Method & Accuracy & Correct & Mistakes & Total Log Loss (bits)\\
		\hline
		MADNB     & 93.17 & 7569 & 555 & 3158.01 \\
		monotone  & 48.44 & 3935 & 4189 & 48494.4 \\
		$1$-CNF   & 48.67 & 3954 & 4170 & 48313 \\
		$2$-CNF   & 93.27 & 7577 & 547 & 6156.91 \\
		$3$-CNF   & \bf{98.58} & \bf{8009} & \bf{115} & \bf{844.51} \\
		$4$-CNF   & 97.49 & 7920 & 204 & 1645.48 \\
		\hline
		\end{tabular}
	\caption{Results on the Mushroom Data Set.}
	\label{tbl:uci_results}
\end{table}

\begin{table}[h!]
	\centering
		\begin{tabular}{| l || c | c | c | c |}
		\hline
		Method & Accuracy & Correct & Mistakes & Total Log Loss (bits)\\
		\hline
		MADNB     & \bf{88.36} & \bf{2824} & \bf{372} & \bf{1366.06} \\
		monotone  & 52.07 & 1664 & 1532 & 15616.3 \\
		$1$-CNF   & 54.57 & 1744 & 1452 & 14836.1 \\
		$2$-CNF   & 80.51 & 2573 & 623 & 6311.06 \\
		$3$-CNF   & 85.58 & 2735 & 461 & 4311.65 \\
		$4$-CNF   & 77.47 & 2476 & 720 & 6848.97 \\
		\hline
		\end{tabular}
	\caption{Results on the King-Rook vs King-Pawn Dataset.}
	\label{tbl:uci_results2}
	\vspace{-1em}
\end{table}

\section{Discussion}

In terms of our methods practical utility, we envision our technique being most useful as component of a larger predictive ensemble.
To give a concrete example, consider the statistical data compression setting, where the cumulative log-loss under some probabilistic model directly corresponds to the size of a file encoded using arithmetic encoding \citep{Witten87}.
Many strong statistical data compression techniques work by adaptively combining the outputs of many different probabilistic models.
For example, the high performance PAQ compressor uses a technique known as geometric mixing \citep{Mattern:2013:LGM:2495257.2495873}, to combine the outputs of many different contextual models in a principled fashion.
Adding our technique to such a predictive ensemble would give it the property that it could exploit $k$-CNF structure in places where it exists.

\section{Conclusion}

This paper has provided two efficient, low-loss algorithms for probabilistically predicting targets generated by some unknown $k$-CNF Boolean function of $d$ Boolean variables in time (for fixed $k$) polynomial in $d$.
The construction of Algorithm~\ref{alg:monotone} is technically interesting in the sense that it is a hybrid Bayesian technique, which performs full Bayesian inference only on the positive examples, with a prior carefully chosen so that the loss suffered on negative examples is kept small.
This approach may be potentially useful for more generally applying the ideas behind Bayesian inference or exponential weighted averaging in settings where a direct application would be computationally intractable.
The more practical Algorithm~\ref{alg:minplus} is less interpretable, but has $O(d)$ space complexity and a per instance time complexity of $O(d)$, while enjoying a loss within a multiplicative $\log n$ factor of the intractable Bayesian predictor using a uniform prior.

\acks{
The authors would like to thank the following people:
Brendan McKay for providing the proof of Theorem \ref{thm:upowhard};
Kee Siong Ng, for the suggestion many years ago to investigate the class of $k$-CNF formulas from an online, probabilistic perspective; and finally Julien Cornebise and Marc Bellemare for some helpful comments and discussions.
}

\section*{References}\label{sec:Bib}
\addcontentsline{toc}{section}{\refname}
\def\refname{\vspace{-4ex}}
\bibliographystyle{alpha}
\begin{small}

\end{small}

\appendix
\section{MADNB Description}\label{app:madnb}
\newcommand{\madnb}{\text{\sc madnb}}
\newcommand{\nb}{\text{\sc nb}}
\newcommand{\dirichlet}{\text{\sc kt}}
\newcommand{\kt}{\text{\sc kt}}

Here we describe the MADNB method of \cite{Dash02exactmodel} as applied to our setting.
At a high level, the algorithm combines the predictions made by many Naive Bayes classifiers.
Each classifier is described by a set of feature indices $\cS \subseteq\{1,...,d\}$, with the feature values at time $t$ defined to be $\llbracket a^i_t \rrbracket$ for all $i \in \cS$.
The MADNB technique aggregates the predictions made by the $2^d$ possible Naive Bayes classifiers corresponding to some subset of features $\cS\subseteq\{1,...,d\}$.

Our presentation will first describe how the parameters can be learnt online using the KT estimator \citep{krichevsky1981pue} for a Naive Bayes classifier containing only binary features, before moving on to describe the MADNB technique.

\paragraph{Parameter learning.}
Consider a sequence $x_{1:n} \in \{0,1\}^n$ generated by successive Bernoulli trials.
If $a$ and $b$ denote the number of zeroes and ones in $x_{1:n}$ respectively, and $\theta \in [0,1]$ denotes the probability of observing a 1 on any given trial, then $\Pr(x_{1:n}|\theta) = \theta^b (1-\theta)^a$.
One way to construct a distribution over $x_{1:n}$, in the case where $\theta$ is unknown, is to weight over the possible values of $\theta$.
The KT-estimator uses the weighting $w(\theta) := \text{Beta($\tfrac{1}{2}$,$\tfrac{1}{2}$)} = \pi^{-1} \theta^{-1/2}(1-\theta)^{-1/2}$, which assigns the probability
$\kt(x_{1:n}) := \int_0^1 \theta^b(1-\theta)^a w(\theta)\,d\theta$.
This quantity can be efficiently computed online by maintaining the $a$ and $b$ counts incrementally and using the chain rule, that is,
$\Pr(x_{n+1}=1|x_{1:n})=1-\Pr(x_{n+1}=0|x_{1:n}) = (b+1/2)/(n+1)$.

\paragraph{Online Naive Bayes Classifier.}
Given binary features described by the indices in $\cS$, if we define
\begin{equation*}
  \nb_\cS(x_{1:t}, a_{1:t})
  ~=~ \dirichlet(x_{1:t})
  \prod_{i \in \cS} \dirichlet_0(a_{1:t}^i | x_{1:t}) \dirichlet_1(a_{1:t}^i | x_{1:t})
  \prod_{i \not\in \cS} \dirichlet(x_{1:t})
\end{equation*}
the probabilistic prediction made by an online Naive Bayes model at time $t$ is given by
\begin{equation*}
\nb_\cS(x_t | x_{<t} ; a_{1:t} ) ~:=~
\frac{\nb_\cS(x_{1:t}, a_{1:t})}{\sum_{y \in \cB} \nb_\cS(x_{<t}y, a_{1:t})},
\end{equation*}
where $\dirichlet_j(a_{1:t}^i | x_{1:t}) := \kt\left(f_j(x_{1:t},
a^i_{1:t})\right)$ and $f_j(x_{1:t}, a^i_{1:t})$ returns the
subsequence of $a^i_{1:t}$ described by the indices $1 \leq k \leq
t$ satisfying $x_k = j$ (of length $a$ for $j=0$ and length $b$ for $j=1$).

\paragraph{MADNB.}
The uniform Bayesian mixture over all $\cS\subseteq\{1,...,d\}$
is defined as
\begin{equation*}
  \nb_*(x_{1:t}, a_{1:t}) ~:=~ \sum_{\cS \in \cP_d} 2^{-d} \, \nb_\cS(x_{1:t}, a_{1:t})
\end{equation*}
The MADNB predictor can now be defined, by at time $t$ assigning
the probability
\begin{equation}\label{eq:madnb_defn}
\madnb(x_t | x_{<t} ; a_{1:t}) ~:=~ \frac{ \nb_*(x_{1:t}, a_{1:t}) }{ \sum_{y \in \cB} \nb_*(x_{<t}y, a_{1:t}) }.
\end{equation}
Equation \ref{eq:madnb_defn} can be computed efficiently, using the identity
\begin{equation*}
  \nb_*(x_{1:t}, a_{1:t})
  ~=~ \dirichlet(x_{1:t}) \prod_{i=1}^d \left[ \frac{1}{2} \dirichlet(a_{1:t}^i) + \frac{1}{2} \dirichlet_0(a_{1:t}^i | x_{1:t}) \dirichlet_1(a_{1:t}^i | x_{1:t}) \right].
\end{equation*}
Note that the MADNB techniques differs from a pure generative (online as well as offline) Bayesian approach, in the sense that
\begin{equation*}
\sum_{\cS \in \cP_d} \frac{1}{2^d} \prod_{t=1}^n \nb_\cS(x_t | x_{<t} ; a_{1:t} )
~\neq~
\prod_{t=1}^n \madnb(x_t | x_{<t} ; a_{1:t})
~\neq~
{\nb_*(x_{1:t}, a_{1:t})\over\sum_{a_{1:n}}\nb_*(x_{1:t}, a_{1:t})}
\end{equation*}
but has the advantage of being computationally tractable.

\section{List of Notation}\label{app:Notation}

\begin{tabbing}
  \hspace{0.13\textwidth} \= \hspace{0.73\textwidth} \= \kill
  {\bf Symbol }      \> {\bf Explanation}                                                                        \\[0.5ex]
  $\cB$              \> $\{\text{false},\text{true}\}=\{\bot,\top\}=\{0,1\}$                                     \\[0.5ex]
  $\llbracket\cB\rrbracket$ \> is 1 if $\cB$ is true, and 0 if $\cB$ is false                                        \\[0.5ex]
  $i$                \> side information vector index $\in\{1,\dots,d\}$                                         \\[0.5ex]
  $t$                \> data item index $\in\{1,...,n\}$                                                         \\[0.5ex]
  $\text{$k$-CNF}$   \> $k$-Conjunctive Normal Form                                                              \\[0.5ex]
	$d$                \> side information dimensionality                                                          \\[0.5ex]
  $a$                \> side information $\in\{0,1\}^d$                                                          \\[0.5ex]
  $a_t^{i}$          \> Boolean side information $i$ at time $t$                                                 \\[0.5ex]
  $x_t\in\cB$        \> binary label at time $t$                                                                 \\[0.5ex]
  $\cS$              \> subset of $\{1,...,d\}$; indices of positive literals in monotone conjunction             \\[0.5ex]
  $h_\cS(a_t)$       \> monotone conjunction hypothesis $\bigwedge_{i\in\cS} a_t^i$                              \\[0.5ex]
  $\cS_n'$           \> the MAP model at time $n$, where $\cS_n' \subseteq\{1,\dots,d\}$                         \\[0.5ex]
  $\xi_d^\a$         \> Bayesian mixture over monotone conjunctions for positive data                            \\[0.5ex]
  $\xi_d^+$          \> heuristic Bayesian predictor for arbitrary data that learns only from positive data      \\[0.5ex]
  $\zeta_d$          \> Algorithm \ref{alg:monotone}, the hybrid Bayesian+memorization predictor                 \\[0.5ex]
	$\pi_d$            \> Algorithm \ref{alg:minplus}, the memory efficient monotone conjunction predictor         \\[0.5ex]
  $\nu_\cS$          \> deterministic predictor corresponding to monotone conjunction $h_\cS$                    \\[0.5ex]
  $w(\cS)$           \> prior weight of the monotone conjunction defined by $\cS \subseteq\{1,...,d\}$           \\[0.5ex]
  $\a$               \> hyper-parameter controlling bias toward smaller/larger formulas                          \\[0.5ex]
  $\e$               \> base of natural logarithm                                                                \\[0.5ex]
  $\log$             \> binary logarithm                                                                         \\[0.5ex]
  $\ln$              \> natural logarithm                                                                        \\[0.5ex]
	$\cL_n(\rho)$      \> cumulative log-loss of $\rho$ at time $n$                                         \\[0.5ex]
	$\cC^k_d$          \> the class of $k$-CNF Boolean formulas that can be formed from $d$ variables              \\[0.5ex]
	$\cA$              \> the incrementally maintained set of negative examples in Algorithm \ref{alg:monotone}    \\[0.5ex]
	$\cD_t$            \> indices of variables not ruled out from the positive examples occurring before time $t$ \\[0.5ex]
	$\cP_d$            \> the powerset of $\{ 1, \dots, d \}$                                                      \\[0.5ex]
	$\kt$              \> KT estimator                                                                             \\[0.5ex]
	$\nb_\cS$          \> Naive Bayes predictor, using feature set $\cS$                                           \\[0.5ex]
	$\madnb$           \> Model averaging discriminative naive Bayes predictor                                     \\[0.5ex]
\end{tabbing}

\end{document}